\documentclass{article}

\PassOptionsToPackage{sort, numbers, compress}{natbib}

\usepackage[preprint]{neurips_2021}

\usepackage[utf8]{inputenc} 
\usepackage[T1]{fontenc}    
\usepackage{hyperref}       
\usepackage{url}            
\usepackage{booktabs}       
\usepackage{amsfonts}       
\usepackage{nicefrac}       
\usepackage{microtype}      
\usepackage{xcolor}         

\usepackage{amsmath}
\usepackage{amssymb}
\usepackage{amsthm}
\usepackage[pdftex]{graphicx}
\usepackage{wrapfig}

\title{An Ensemble Approach Towards Adversarial Robustness}

\author{%
  Haifeng Qian \\
  IBM T. J. Watson Research Center\\
  Yorktown Heights, NY, USA\\
  \texttt{qianhaifeng@us.ibm.com} \\
}

\begin{document}

\maketitle

\begin{abstract}
  It is a known phenomenon that adversarial robustness comes at a cost to natural accuracy.
  To improve this trade-off, this paper proposes an ensemble approach that divides a complex robust-classification task into simpler subtasks.
  Specifically, \emph{fractal divide} derives multiple training sets from the training data, and \emph{fractal aggregation} combines inference outputs from multiple classifiers that are trained on those sets.
  The resulting ensemble classifiers have a unique property that ensures robustness for an input if certain don't-care conditions are met.
  The new techniques are evaluated on MNIST and Fashion-MNIST, with no adversarial training.
  The MNIST classifier has 99\% natural accuracy, 70\% measured robustness and 36.9\% provable robustness, within $L_2$ distance of 2.
  The Fashion-MNIST classifier has 90\% natural accuracy, 54.5\% measured robustness and 28.2\% provable robustness, within $L_2$ distance of 1.5.
  Both results are new state of the art, and we also present new state-of-the-art binary results on challenging label-pairs.
\end{abstract}

\section{Introduction} \label{sec:intro}

Adversarial examples \cite{szegedy,goodfellow} pose both theoretical questions on how machine-learning models generalize and practical challenges in AI applications.
One can fool facial recognition and impersonate another person by wearing adversarial eyeglass frames \cite{glassframe}; with four small stickers on it, a stop sign gets recognized as a ``speed limit 45'' sign \cite{stopsign}; a Burger King sign got recognized by a Tesla as a stop sign \cite{burger}.
Attacks have also been demonstrated on speech signal \cite{carlini2018audio} and natural language \cite{alzantot2018}.

Adversarial robustness is an extraordinarily difficult problem.
Many proposals have failed when faced with strong attacks \cite{carlini2017,athalye2018,tramer2020}.
One notable success is that an MNIST model achieves human-level robustness with respect to the $L_{\infty}$ metric using adversarial training \cite{mit}.
$L_{\infty}$, however, is not the only metric of interest.
$L_2$ is just as important if not more.
In fact, the aforementioned eyeglass frames, stickers and Burger King sign are all $L_2$ attacks: the distortion has a small $L_2$ norm and a large $L_{\infty}$ norm, and therefore models with $L_{\infty}$ robustness cannot defend against them.
$L_2$ robustness is less understood and perhaps more difficult: for Fashion-MNIST, robustness within $L_2$ distance of 0.88 guarantees robustness within $L_{\infty}$ distance of 8/255, yet $L_{\infty}$ robustness implies little $L_2$ robustness.

A major obstacle is the robustness-accuracy trade-off that seems to exist in all approaches, including adversarial training \cite{tradeoff}, adversarial polytope \cite{wong2018scaling} and $L_2$-nonexpansive neural network (L2NNN) \cite{l2nnn}.
For example, L2NNN achieves substantial provable $L_2$ robustness, yet its natural accuracy on MNIST drops to 98.2\%.
A later work of \cite{nbr} divides a task into simpler subtasks, each of which is solved by an L2NNN with more favorable robustness-accuracy trade-off; it was demonstrated on an MNIST 4-9 classifier that has state-of-the-art $L_2$ robustness with little loss of accuracy.

This paper shares the same strategy of divide-and-conquer as \cite{nbr} and builds on top of its architecture.
A new technique of \emph{fractal divide} replaces how subtasks are formed, and a new technique of \emph{fractal aggregation} replaces how the outputs of subtask classifiers are combined.
We demonstrate new state of the art in $L_2$ robustness on MNIST and Fashion-MNIST, as well as on binary classification.

In principle, fractal divide and fractal aggregation can work with any metric of interest, as long as each subtask classifier has bounded Lipschitz constant with respect to the metric.
This paper focuses on $L_2$ due to the availability of L2NNN.
It's also possible to build each subtask classifier by adversarial training, and we choose L2NNN over adversarial training for its provable guarantees.


\section{Background} \label{sec:bg}



The goal of robust classification is to predict the correct label not only at a data point but also within a ball of a certain radius around it.
The achievable radius depends on the separation of data.
Let $d\left({\bf t}\right)$ be the minimum distance from input ${\bf t}$ to any point with a different label.
$d\left({\bf t}\right)/2$ is referred to as the \emph{oracle robustness} in \cite{nbr}, which reported the statistics on the MNIST training set: the oracle robustness radius is above 2 for 96\% of the images, above 2.5 for 79\% and above 3 for 51\%.
Therefore, $L_2$ radius of 2 is a meaningful threshold to measure MNIST classifiers.
We perform the same measurement on Fashion-MNIST.
The oracle radius is above 1.5 for 95\% of the images, above 2 for 66\% and above 2.5 for 34\%.
Therefore, $L_2$ radius of 1.5 is a meaningful threshold.



L2NNNs use a combination of regularization and architecture choices to ensure that the Lipschitz constant of a network is no greater than 1 \cite{l2nnn}.
Consequently, if the distortion to input is limited by $L_2$ norm of $\varepsilon$, the output won't change by more than $\varepsilon$, and this leads to robustness.
This also enables poorman's adversarial training \cite{nbr}: one may simply increase or decrease L2NNN outputs by an adversarial offset to simulate the worst-case scenario of an attack.
A limitation of L2NNN is the robustness-accuracy trade-off.
The reason is that a task of robust classification is often too complex for a monolithic L2NNN, and this explains why \cite{nbr} achieves a better trade-off by dividing a task.


Neural belief reasoner (NBR) \cite{nbr} is a method for unsupervised learning and it has a variant for classification with reduced complexity.
The variant is an ensemble of $K$ subtask classifiers, which differ from each other by using different frames of discernment or by using different subsets of the training data.
If all subtask classifiers are binary, the overall complexity is linear with respect to $K$.
For example, \{1,7,9\} \{3,5,8\} is a binary frame of discernment for MNIST, and the subtask classifier computes a belief value $b_\mathrm{\scriptscriptstyle I}\in\left[0,1\right]$ for the possible world of $\mathrm{label}\in\left\{0,2,3,4,5,6,8\right\}$ and a belief value of $b_\mathrm{\scriptscriptstyle II}\in\left[0,1\right]$ for the possible world of $\mathrm{label}\in\left\{0,1,2,4,6,7,9\right\}$.
In plain English, the subtask classifier states that, for a given input, it has belief $b_\mathrm{\scriptscriptstyle I}$ that the label is not one of \{1,7,9\} and has belief $b_\mathrm{\scriptscriptstyle II}$ that the label is not one of \{3,5,8\}.
With various frames of discernment, an NBR can gather sufficient knowledge for the overall multiclass classification task, and it reasons about knowledge from subtask classifiers using belief functions \cite{ds}.

This paper follows the same architecture with all binary frames of discernment, and adds two new techniques.
\emph{Fractal divide} replaces how subsets of the training data were formed.
\emph{Fractal aggregation} replaces how subtask classifiers with identical frames of discernment were combined.
We also introduce a new way to aggregate over different frames of discernment, at the end of Section~\ref{sec:aggr}. 
The overall complexity remains linear with respect to $K$.

Each subtask classifier is implemented as a scalar-output L2NNN with a sigmoid unit added at the end.
We'll now list equations of its outputs, i.e., $b_\mathrm{\scriptscriptstyle I}$ and $b_\mathrm{\scriptscriptstyle II}$. 
These are after adjustments for sizes of the training subsets and after dynamic scaling as described in \cite{nbr}: for the $i^\mathrm{th}$ subtask, $1 \leq i \leq K$,
\begin{align}
  b_{\mathrm{\scriptscriptstyle I},i} =&
  \begin{cases}
    \frac{2 \cdot b_i \cdot f_{\mathrm{\scriptscriptstyle I},i} \cdot \left(0.5-x_i\right)}
         {1-0.5 \cdot b_i+b_i \cdot f_{\mathrm{\scriptscriptstyle I},i} \cdot \left(0.5-x_i\right)}, & \text{if } x_i < 0.5 \label{eq:b1}\\
    0, & \text{otherwise}
  \end{cases}\\
  b_{\mathrm{\scriptscriptstyle II},i} =&
  \begin{cases}
    \frac{2 \cdot b_i \cdot f_{\mathrm{\scriptscriptstyle II},i} \cdot \left(x_i-0.5\right)}
         {1-0.5 \cdot b_i+b_i \cdot f_{\mathrm{\scriptscriptstyle II},i} \cdot \left(x_i-0.5\right)}, & \text{if } x_i > 0.5 \label{eq:b2}\\
    0, & \text{otherwise}
  \end{cases}\\
  \mathrm{where}\,&\, x_i\triangleq\mathrm{sigmoid}\left(s_i\cdot G_i\left({\bf t}\right)\right) \label{eq:x}
\end{align}
where ${\bf t}$ is the input datum; $G_i\left(\cdot\right)$ is the L2NNN; $s_i>0$ and $0\leq b_i\leq 1$ are two trainable scalar parameters; $f_{\mathrm{\scriptscriptstyle I},i}$ is the fraction of training data used by this subtask with the first label group, 
and $f_{\mathrm{\scriptscriptstyle II},i}$ is that for the second label group.
For example, if the $i^\mathrm{th}$ subtask has the frame of discernment of \{1,7,9\} \{3,5,8\} and the classifier is trained on 30\% of the data with labels \{1,7,9\} and 60\% of the data with labels \{3,5,8\}, then $f_{\mathrm{\scriptscriptstyle I},i}=0.3$ and $f_{\mathrm{\scriptscriptstyle II},i}=0.6$.
Note that one of (\ref{eq:b1})(\ref{eq:b2}) is zero for any ${\bf t}$.

\section{Fractal divide and fractal aggregation} \label{sec:tech}

A hypothesis in this work is that robustness comes from reasoning.
In classical logic, the truth value of a formula remains unchanged when the truth values of some predicates in it are fixed while the others are don't-cares.
To emulate this phenomenon in classification, two things are needed.
The first is a set of predicates that represent diverse knowledge and that are not easy to change, and this is achieved by fractal divide in Section~\ref{sec:divide} and the L2NNN implementation of individual subtasks.
The second is a formula that combines the predicates and that incorporates uncertainty, and this is achieved by fractal aggregation in Section~\ref{sec:aggr}.

\subsection{Overview of the ensemble method} \label{sec:ensemble}

The presentation focuses on binary classification, where subtasks have the same binary frame of discernment but different subsets of the training data.
The training subsets are formed by fractal divide.
After subtask classifiers are trained, fractal aggregation combines them at inference time.

Multiclass classification is the next-level ensemble: it's an ensemble of binary classifiers with different frames of discernment, each of which may be an ensemble itself.
The unique aspect is how to aggregate across frames of discernment, and this is addressed at the end of Section~\ref{sec:aggr}.

L2NNNs, i.e., $G_i\left(\cdot\right)$'s in (\ref{eq:x}), are trained in the same way as in \cite{nbr} with poor man's adversarial training.
The only difference is that we do not have memorization subtasks, which were used in \cite{nbr} to memorize training data points that are difficult to classify robustly and that are excluded from subsets for training L2NNNs.
The strategy worked for MNIST but it might not generalize well in other tasks.
Instead, in this paper we include the difficult data points in the training subsets but use a smaller adversarial offset for them in poor man's adversarial training.
With this simple trick, we avoid memorization, reduce $K$, and the L2NNNs have more training data.

To train parameters $s_i$ and $b_i$, $1 \leq i \leq K$, in (\ref{eq:b1})(\ref{eq:b2})(\ref{eq:x}), a fraction of training data need to be reserved for them and not used for training L2NNNs.
In this paper, we choose a simple approach of uniform $s_i$'s and uniform $b_i$'s, -- treating them as two hyperparameters rather than $2K$ trainable parameters, -- and the entire training data are used for L2NNNs.
We leave training $s_i$'s and $b_i$'s to future work.

\subsection{Fractal divide} \label{sec:divide}

For clarity, we will refer to the training data for binary classification as data with label I and label II, even though label I/II may correspond to a group of labels, e.g., \{1,7,9\} in the earlier example.

\begin{wrapfigure}{r}{0.35\columnwidth}
  \centering
  \includegraphics[width=0.3\columnwidth]{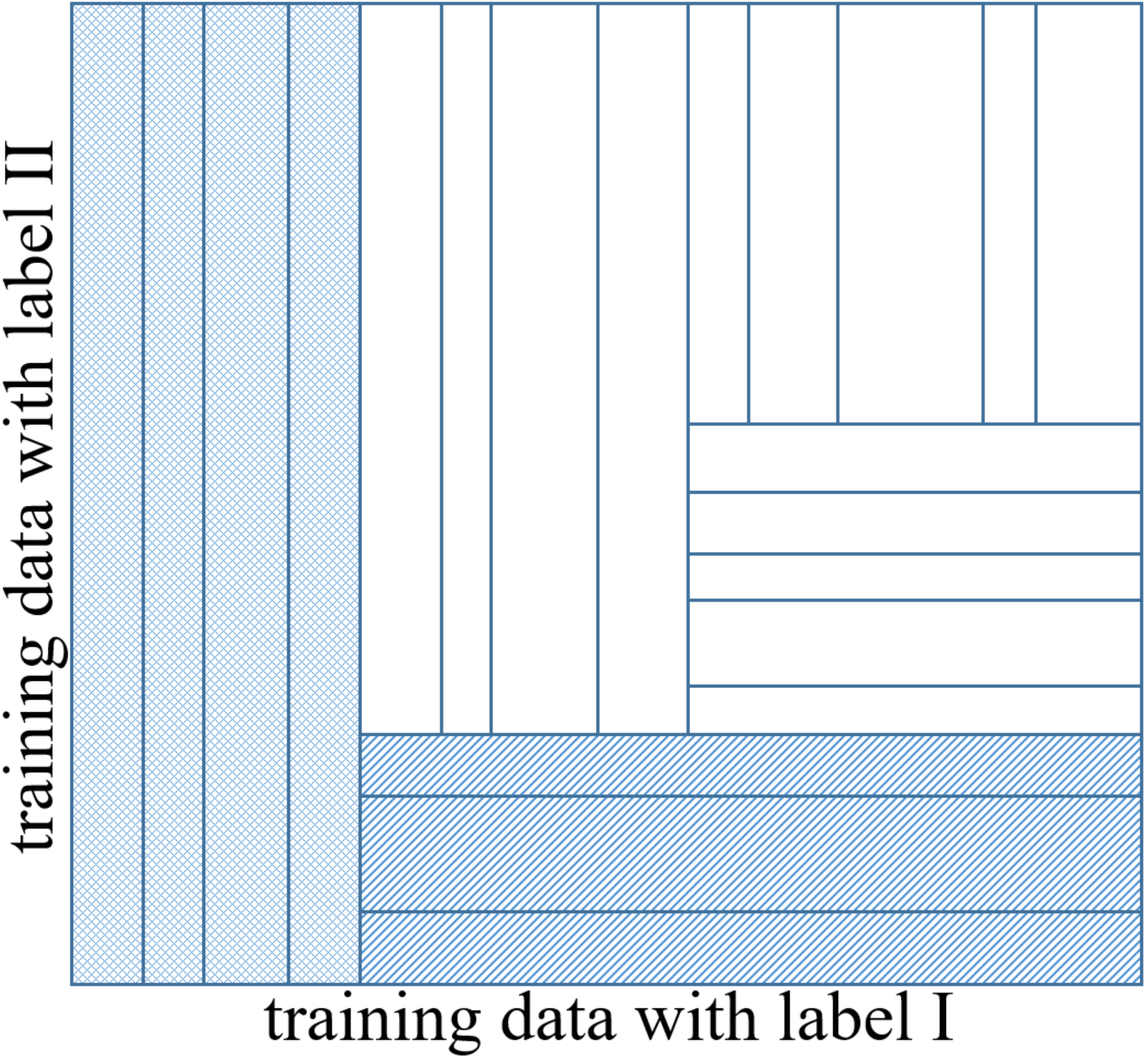}
  \caption{Fractal divide with 5 levels. First 2 levels are highlighted.}
  \label{fig:divide}
  \vspace{-10pt}
\end{wrapfigure}

Figure~\ref{fig:divide} illustrates an example of fractal divide.
Consider the figure as a two-dimensional grid where each x-coordinate corresponds to a datum with label I and each y-coordinate corresponds to a datum with label II.
Each rectangle specifies one training subset: its horizontal span is a subset of data with label I and its vertical span is a subset of data with label II.
There are 21 rectangles in Figure~\ref{fig:divide} and they collectively cover the entire region with no overlap.
In other words, any pair of training data with opposite labels appears in one and only one of the training subsets, and no data pair gets contrasted twice.

We refer to the four rectangles on the left with grid shade as the first-level training subsets.
Each of them is a subtask to distinguish all data with label II against a different subset of data with label I.
We refer to the three lower-right rectangles with stripe shade as the second-level training subsets.
Each of them is a subtask to distinguish the same set of data with label I, which are ones excluded from the first level, against a different subset of data with label II.
This goes on, and Figure~\ref{fig:divide} illustrates a five-level fractal structure.
In theory, fractal divide can have an arbitrary number of levels.
The limit is when subsets in later levels become so small that the trained subtask classifiers generalize poorly.

In the MNIST and Fashion-MNIST classifiers, we use three-level fractal structures.
Building fractal structures requires partitioning training data across levels and within each level.
\begin{itemize}
\item
  For partitioning within a level, subtask classifiers are trained jointly and training data is periodically re-partitioned.
  If this level partitions data with label I, each datum ${\bf t}$ is assigned to the classifier with maximum $G_i\left({\bf t}\right)$, i.e., the classifier with the most robust prediction on ${\bf t}$.
  For partitioning data with label II, the choice is minimum $G_i\left({\bf t}\right)$ instead.
\item
  For partitioning across levels, we first train two one-level fractal structures: the first has $K_1$ classifiers that split data with label I, and the second has $K_2$ classifiers that split data with label II.
  We use the $G_i\left({\bf t}\right)$ values of the former as $K_1$ features and use the k-means algorithm to partition data with label I into two parts, and use the $K_2$ $G_i\left({\bf t}\right)$ values of the latter to partition data with label II into two parts.
  There are eight possible three-level fractal structures.
  Figure~\ref{fig:3lvl} illustrates two, and the other six can be formed by swapping the two parts for label I and/or II.
  One may pick one or train all eight and pick the best.
\end{itemize}

\begin{wrapfigure}{r}{0.5\columnwidth}
  \centering
  \includegraphics[width=0.45\columnwidth]{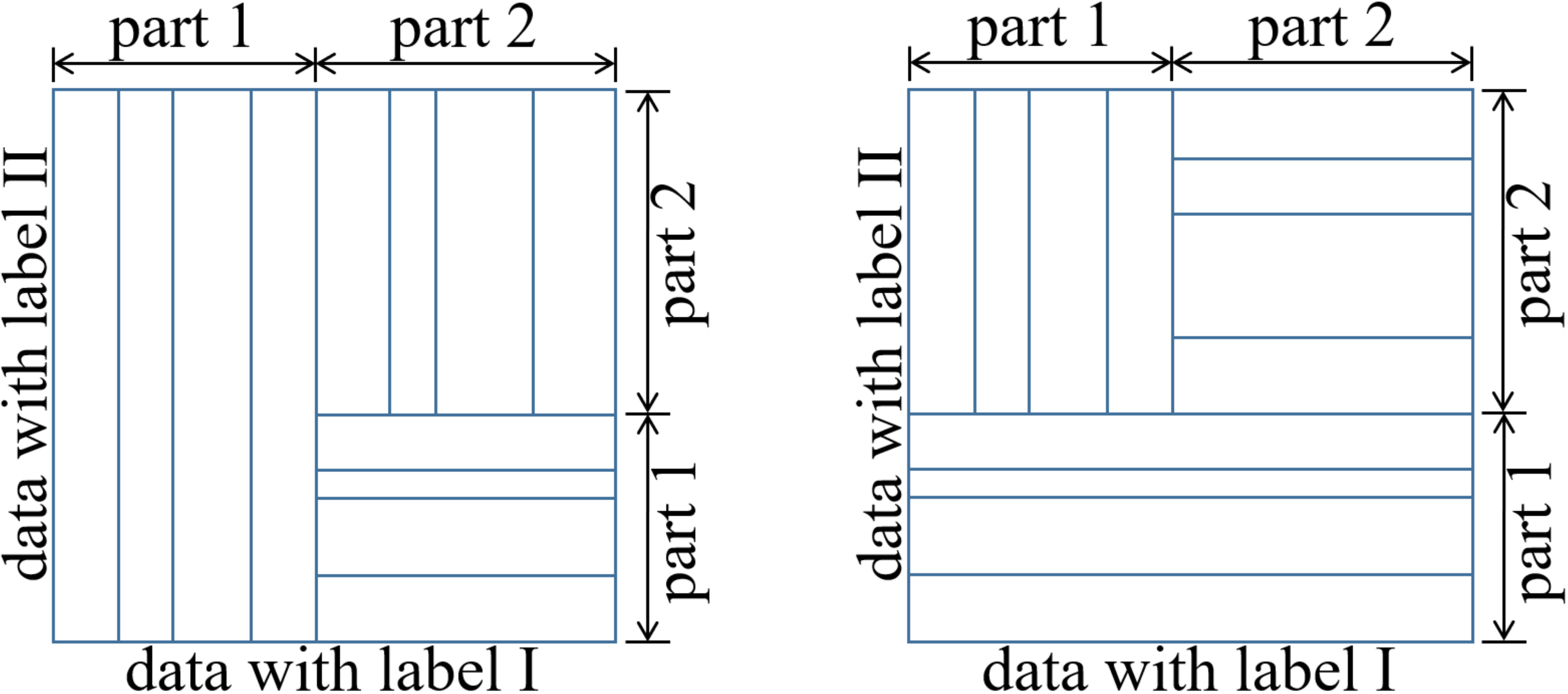}
  \caption{Examples of fractal divide with 3 levels.}
  \label{fig:3lvl}
  \vspace{10pt}
  \includegraphics[width=0.35\columnwidth]{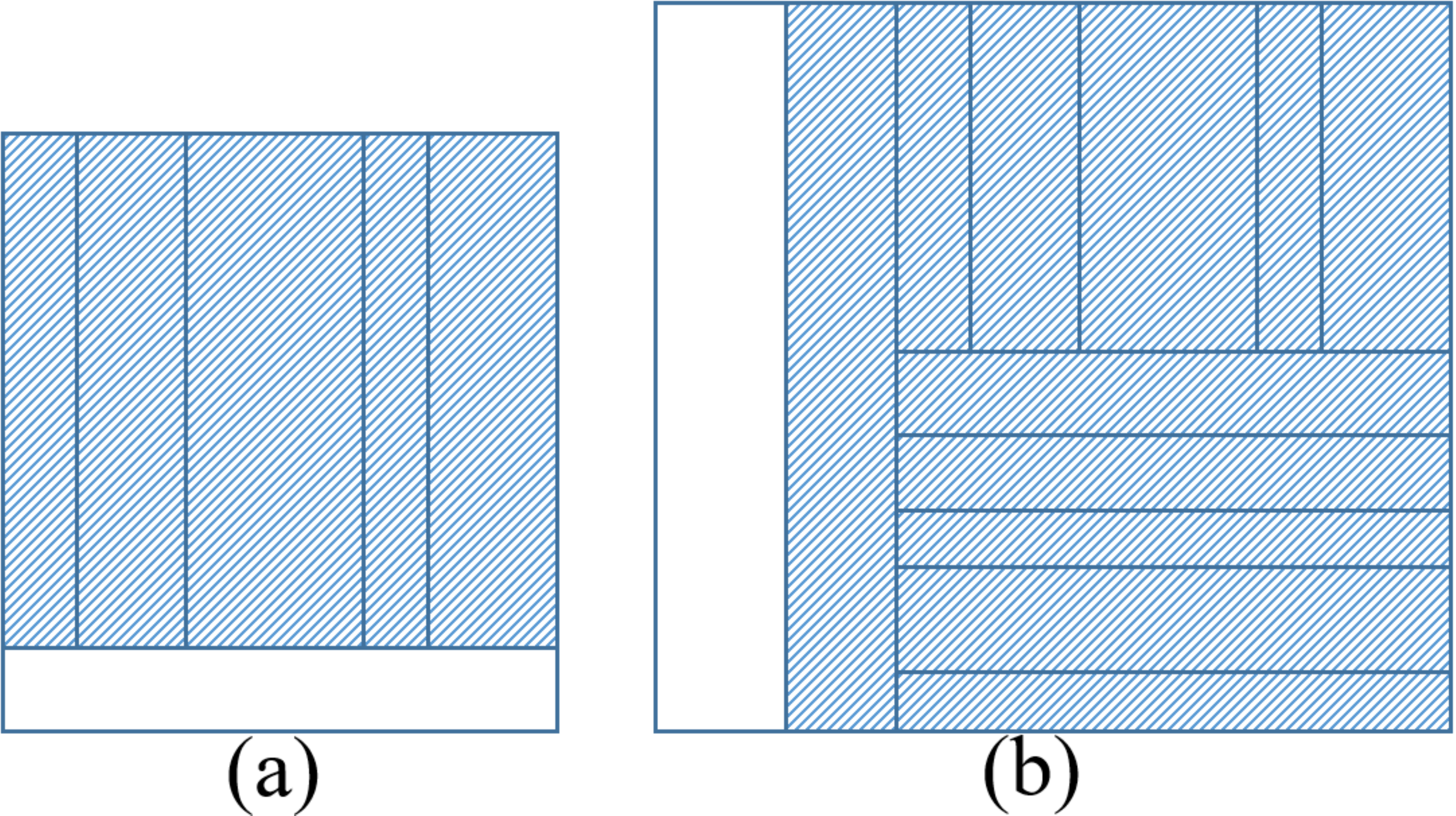}
  \caption{Fractal aggregation's two atomic operations. Each operation combines two classifiers that are represented by the white rectangle and the shaded rectangle. The shaded one may be the result of previous atomic operations.}
  \label{fig:aggr}
  \vspace{-30pt}
\end{wrapfigure}

The intuition is that subtasks should work on data with different characteristics and thereby learn different knowledge.
It's also possible to update cross-level partitioning periodically during training.
When more than three levels are needed, we could partition data recursively.

\subsection{Fractal aggregation} \label{sec:aggr}

To combine the outputs (\ref{eq:b1})(\ref{eq:b2}) of subtask classifiers, let's first perform a simple transformation:
\begin{align}
  w_{\mathrm{\scriptscriptstyle I},i}  =& \log\left(1-b_{\mathrm{\scriptscriptstyle I},i}\right) \label{eq:w1}\\
  w_{\mathrm{\scriptscriptstyle II},i} =& \log\left(1-b_{\mathrm{\scriptscriptstyle II},i}\right) \label{eq:w2}
\end{align}
Intuitively, $w_{\mathrm{\scriptscriptstyle I},i}$ is the negative of weight of evidence \cite{ds} against label I from the $i^\mathrm{th}$ subtask classifier, and $w_{\mathrm{\scriptscriptstyle II},i}$ is that against label II.
The upper limit of a $w$ value is zero, which means that there is no evidence against a label.

Fractal aggregation retraces the structure built by fractal divide: the subtasks in the last level are first combined into one classifier, and then it gets combined with subtasks in the second last level, and so on.
The process can be viewed as a sequence of two atomic operations that are illustrated in Figure~\ref{fig:aggr}.
The first operation combines two classifiers that have been trained on the same subset of data with label I and two disjoint subsets of data with label II.
The second operation is the opposite.
Pictorially, the first operation merges the two rectangles in Figure~\ref{fig:aggr}(a) into one rectangle, while the second operation does so on Figure~\ref{fig:aggr}(b).
Each atomic operation produces a combined classifier that becomes one of the two operands in the next atomic operation.

To derive the atomic operations, we view an operand as two bodies of evidence: some evidence against label I and some evidence against label II.
In the first atomic operation, the evidence against label I from both operands has been derived from the same subset of data with label I, and therefore they are two overlapping bodies of evidence and the combined weight of evidence against label I is the maximum of the two weights.
On the other hand, the evidence against label II from the two operands has been derived from two disjoint subsets of data with label II, and hence the two bodies of evidence against label II can be combined using Dempster's rule of combination \cite{ds}: the resulting weight of evidence is simply the sum of the two weights.
In summary, the first atomic operation is
\begin{equation}
  w_{\mathrm{\scriptscriptstyle I},\mathrm{combined}}  = \min\left(w_{\mathrm{\scriptscriptstyle I},i},w_{\mathrm{\scriptscriptstyle I},j}\right) \,\,,\,\,
  w_{\mathrm{\scriptscriptstyle II},\mathrm{combined}} = w_{\mathrm{\scriptscriptstyle II},i} + w_{\mathrm{\scriptscriptstyle II},j}
  \label{eq:aggr}
\end{equation}
where we use $i,j$ to index operands for brevity, even though one of the two operands may be an ensemble from previous operations and not one of the original subtasks.
(\ref{eq:aggr}) uses min instead of max because $w$'s are the negative of weights of evidence.
The second atomic operation is the opposite:
\begin{equation}
  w_{\mathrm{\scriptscriptstyle I},\mathrm{combined}}  = w_{\mathrm{\scriptscriptstyle I},i} + w_{\mathrm{\scriptscriptstyle I},j} \,\,,\,\,
  w_{\mathrm{\scriptscriptstyle II},\mathrm{combined}} = \min\left(w_{\mathrm{\scriptscriptstyle II},i},w_{\mathrm{\scriptscriptstyle II},j}\right)
  \label{eq:aggr22}
\end{equation}

It is also possible to write fractal aggregation as nested functions.
For example, the final outputs of fractal aggregation for the first fractal structure in Figure~\ref{fig:3lvl} are
\begin{align}
  w_\mathrm{\scriptscriptstyle I}  =& \sum_{i=1}^4 w_{\mathrm{\scriptscriptstyle I},i} + \min\left(\min_{i=5}^8 w_{\mathrm{\scriptscriptstyle I},i},\sum_{i=9}^{12} w_{\mathrm{\scriptscriptstyle I},i} \right) \label{eq:nest1}\\
  w_\mathrm{\scriptscriptstyle II} =& \min\left(\min_{i=1}^4 w_{\mathrm{\scriptscriptstyle II},i}, \sum_{i=5}^{8} w_{\mathrm{\scriptscriptstyle II},i} + \min_{i=9}^{12} w_{\mathrm{\scriptscriptstyle II},i} \right) \label{eq:nest2}
\end{align}
where the subtasks with indices 1--4 are in the first level, 5--8 in the second level and 9-12 in the third.
(\ref{eq:nest1})(\ref{eq:nest2}) can be viewed as a logic formula with uncertainty in belief-function representation, which meets the second goal stated at the beginning of Section~\ref{sec:tech}.
In plain English, (\ref{eq:nest1}) states that the label is not label I if 1) enough first-level classifiers determine so \emph{and} 2) one of the second-level classifiers determines so \emph{or} enough third-level classifiers determine so.
Equation (\ref{eq:nest2}) has a similar interpretation.

For binary classification, the above are the final outputs.
For multiclass classification, we still need to aggregate over binary classifiers with different frames of discernment.
Again, we view each binary classifier as two bodies of evidence: some evidence against the first label group in its frame of discernment and some evidence against the second label group.
For a label $l$, there are multiple bodies of evidence against it, one from each binary classifier where the frame of discernment involves $l$.
Note that each binary classifier is after fractal aggregation, and hence its evidence against $l$ is derived from all training data with label $l$.
Therefore, the multiple bodies of evidence are overlapping, and the combined weight is the maximum of the weights:
\begin{equation}
  w_{\textrm{label}\,l} = \min\left( \min_{j \mid l\in F_j}w_{\mathrm{\scriptscriptstyle I},\mathrm{bin}\,j} , \min_{j \mid l\in F^\prime_j}w_{\mathrm{\scriptscriptstyle II},\mathrm{bin}\,j} \right) \label{eq:multiclass}
\end{equation}
where ``bin $j$'' denotes one of the binary classifiers after fractal aggregation, and $F_j$ and $F^\prime_j$ denotes the two label groups in its frame of discernment.


\subsection{Don't-care conditions} \label{sec:dc}

Let's again consider binary classification and examine an important property of fractal divide and fractal aggregation: an ensemble can be robust on an input when only a subset of its subtask classifiers are robust on this input.
The rest of subtasks only need to satisfy some weak requirements, or no requirement at all in certain scenarios; they are similar to don't-cares in logic.
For this reason, we refer to achieving a controlling subset of robust subtask classifiers for an input as a \emph{don't-care condition}.

Without loss of generality, let's focus on ensembles where a first-level subtask is trained on all data with label II, for example, Figure~\ref{fig:divide} and the first in Figure~\ref{fig:3lvl}.
The discussion applies to other ensembles by swapping labels I and II.
As mentioned in Section~\ref{sec:ensemble}, we assume uniform $b_i$'s in (\ref{eq:b1})(\ref{eq:b2}).

\newtheorem{theorem}{Theorem}
\begin{theorem} \label{th:1lvl}
  A fractal ensemble classifier classifies an input with label I correctly if there exists a first-level subtask classifier $i$ such that $x_i>0.5$ and that $x_j\geq 1-x_i,\forall j\neq i$.
\end{theorem}

Note that $x$'s are defined by (\ref{eq:x}).
Theorem~\ref{th:1lvl} states a sufficient condition for a fractal classifier to be correct: one first-level subtask classifier is correct and the others are not very wrong.
The second part of the condition becomes trivial and can be dropped if $x_i$ is 1, which is approximately satisfied when $G_i\left({\bf t}\right)$ is sufficiently large and $s_i$ is sufficiently large.
In other words, if just one L2NNN in the first level has enough confidence on input ${\bf t}$ such that the sigmoid in (\ref{eq:x}) is in the saturation region, the ensemble predicts the correct label I regardless of the other L2NNNs.

Theorem~\ref{th:1lvl} has direct implications on robustness.
Due to the nature of L2NNN, an adversarial example ${\bf t^\prime}$ can only reduce $G_i\left({\bf t^\prime}\right)$ from $G_i\left({\bf t}\right)$ by a limited amount.
If $G_i\left({\bf t^\prime}\right)$ is still positive and enough for (\ref{eq:x}) to stay in the saturation region, the fractal classifier is guaranteed to be robust on ${\bf t}$.

To help proving Theorem~\ref{th:1lvl}, we need Lemma~\ref{th:lemma} and its proof is in the appendix.

\newtheorem{lemma}{Lemma}
\begin{lemma} \label{th:lemma}
  For $\alpha_j \geq 0$, $1\leq j \leq n$, such that $\sum_{j=1}^n \alpha_j \leq 1 $, this inequality holds:
  \begin{equation}
    \prod_{j=1}^n \frac{1-\alpha_j}{1+\alpha_j} \geq \frac{1-\sum_{j=1}^n\alpha_j}{1+\sum_{j=1}^n\alpha_j}
  \end{equation}
\end{lemma}

\begin{proof}[Proof of Theorem~\ref{th:1lvl}]
  Applying (\ref{eq:aggr22})(\ref{eq:w2})(\ref{eq:b2}) in sequence, and utilizing the fact that $f_{\mathrm{\scriptscriptstyle II},i}=1$, we have
  \begin{equation}
    w_\mathrm{\scriptscriptstyle II} \leq w_{\mathrm{\scriptscriptstyle II},i} = \log\left(1-b_{\mathrm{\scriptscriptstyle II},i}\right)
    =\log\left(1-\frac{2 \cdot b \cdot \left(x_i-0.5\right)}{1-0.5 \cdot b+b \cdot \left(x_i-0.5\right)}\right)
    =\log\frac{1-b \cdot x_i}{1-b+b \cdot x_i}
    \label{eq:pw2}
  \end{equation}
  We have dropped the subscript from $b_i$ since we assume uniform values.
  Now let's examine $w_\mathrm{\scriptscriptstyle I}$.
  It is a nested function of $w_{\mathrm{\scriptscriptstyle I},j}$'s that is similar to (\ref{eq:nest1}).
  For any specific input ${\bf t}$, however, it is the sum of a subset of $w_{\mathrm{\scriptscriptstyle I},j}$'s.
  Let $\Lambda_{\bf t}$ denote the subtask-indices of this subset, and note that $\Lambda_{\bf t}$ varies for different ${\bf t}$ and that it only includes nonzero $w_{\mathrm{\scriptscriptstyle I},j}$'s.
  Fractal aggregation ensures that the subtasks in $\Lambda_{\bf t}$ have been trained on disjoint subsets of data with label I, and we know that $i\notin \Lambda_{\bf t}$.
  Hence, applying (\ref{eq:w1})(\ref{eq:b1}),
  \begin{align}
    w_\mathrm{\scriptscriptstyle I} =& \sum_{j\in\Lambda_{\bf t}} w_{\mathrm{\scriptscriptstyle I},j}
    = \sum_{j\in\Lambda_{\bf t}}\log\left(1- \frac{2 b f_{\mathrm{\scriptscriptstyle I},j} \left(0.5-x_j\right)}
    {1-0.5 b+b f_{\mathrm{\scriptscriptstyle I},j} \left(0.5-x_j\right)} \right) \\
    \mathrm{where}&\,\, \sum_{j\in\Lambda_{\bf t}}f_{\mathrm{\scriptscriptstyle I},j} \leq 1-f_{\mathrm{\scriptscriptstyle I},i} < 1 \label{eq:flimit}
  \end{align}
  Because $w_\mathrm{\scriptscriptstyle I}$ is an increasing function with respect to each $x_j$, we can replace $x_j$ with the lower bound:
  \begin{equation}
    w_\mathrm{\scriptscriptstyle I} \geq \sum_{j\in\Lambda_{\bf t}}\log\left(1- \frac{2 b f_{\mathrm{\scriptscriptstyle I},j} \left(0.5-1+x_i\right)}
    {1-0.5 b+b f_{\mathrm{\scriptscriptstyle I},j} \left(0.5-1+x_i\right)} \right)
    =\log\prod_{j\in\Lambda_{\bf t}}\frac{1 - \frac{x_i-0.5}{1-0.5b} \cdot b \cdot f_{\mathrm{\scriptscriptstyle I},j}}
    {1 + \frac{x_i-0.5}{1-0.5b} \cdot b \cdot f_{\mathrm{\scriptscriptstyle I},j}}
  \end{equation}
  Applying Lemma~\ref{th:lemma} and (\ref{eq:flimit})(\ref{eq:pw2}), we get
  \begin{equation}
    w_\mathrm{\scriptscriptstyle I} \geq \log\frac{1 - \frac{x_i-0.5}{1-0.5b} \cdot b \cdot\sum_{j\in\Lambda_{\bf t}}f_{\mathrm{\scriptscriptstyle I},j}}
    {1 + \frac{x_i-0.5}{1-0.5b} \cdot b \cdot\sum_{j\in\Lambda_{\bf t}}f_{\mathrm{\scriptscriptstyle I},j}}
    > \log\frac{1 - \frac{x_i-0.5}{1-0.5b} \cdot b \cdot 1}
    {1 + \frac{x_i-0.5}{1-0.5b} \cdot b \cdot 1}
    = \log\frac{1-b \cdot x_i}{1-b+b \cdot x_i} \geq w_\mathrm{\scriptscriptstyle II}
  \end{equation}
\end{proof}

There are more don't-care conditions than Theorem~\ref{th:1lvl}.
Intuitively, Theorem~\ref{th:1lvl} stands because fractal aggregation derives evidence against label II from the entire training data with label II, while it derives evidence against label I from a proper subset of data with label I.
Similar conditions are illustrated in Figure~\ref{fig:dc} for a three-level fractal structure.
More exist for structures with deeper levels.

\begin{wrapfigure}{r}{0.45\columnwidth}
  \centering
  \includegraphics[width=0.4\columnwidth]{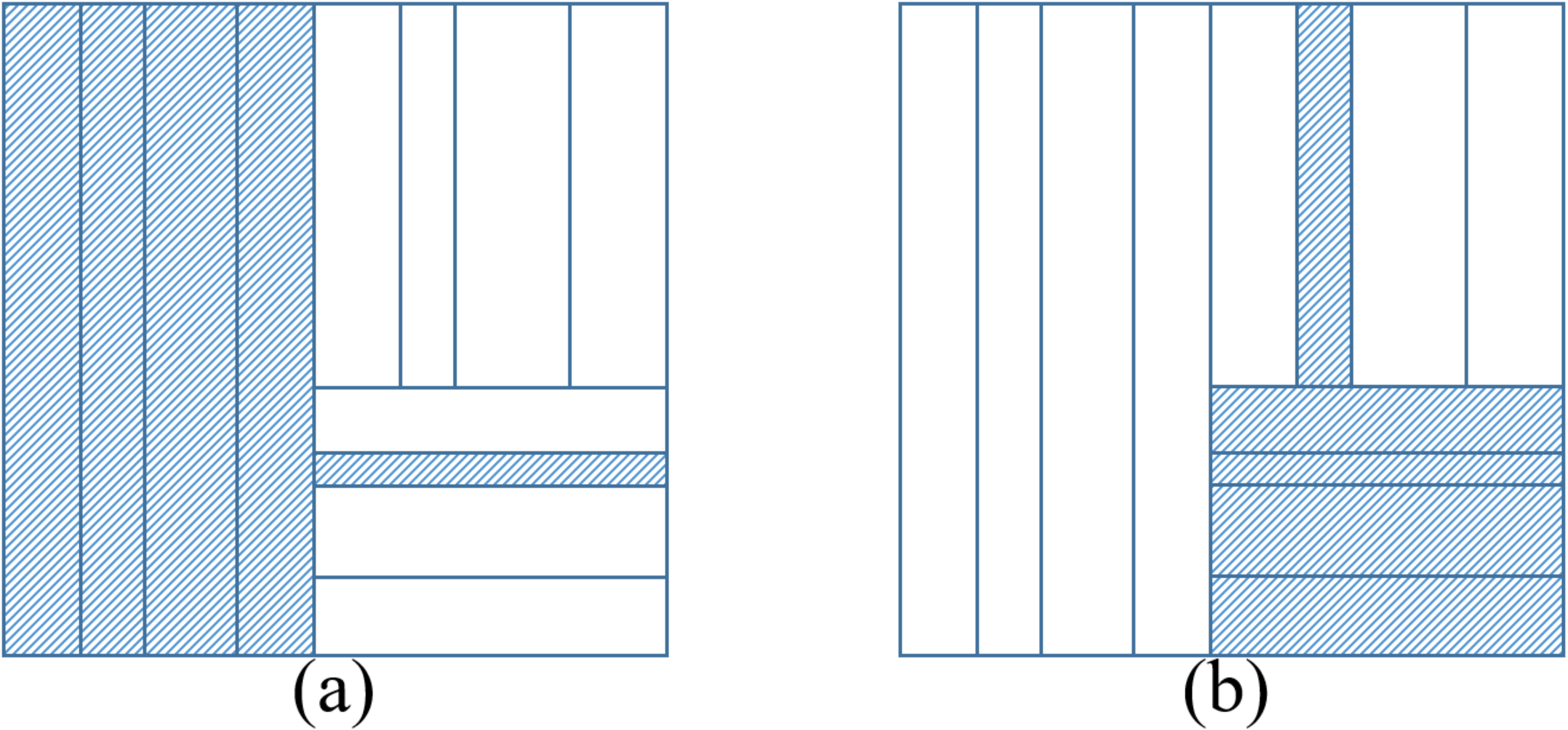}
  \caption{Approximate don't-care conditions for an input with (a) label II and (b) label I. The shaded classifiers need to be correct.}
  \label{fig:dc}
  \vspace{-5pt}
\end{wrapfigure}

Unfortunately, we do not have theorems for conditions in Figure~\ref{fig:dc}.
A proof would boil down to the following inequality, which has the role of Lemma~\ref{th:lemma} for them.
\begin{equation}
  \prod_{j=1}^{n_1} \frac{1-\alpha_j}{1+\alpha_j} > \prod_{j=1}^{n_2} \frac{1-\beta_j}{1+\beta_j} \label{eq:dc}
\end{equation}
where $\alpha_j \geq 0$, $1\leq j \leq n_1$; $\beta_j \geq 0$, $1\leq j \leq n_2$; $\sum_{j=1}^{n_1} \alpha_j < \sum_{j=1}^{n_2} \beta_j \leq 1$.
(\ref{eq:dc}) holds most of the time yet there is no guarantee.
However, it is highly unlikely that an ensemble that satisfies one of the conditions in Figure~\ref{fig:dc} for an input would predict the wrong label: many subtask classifiers would have to be wrong simultaneously, and the fractions $f_{\mathrm{\scriptscriptstyle I},j}$'s and $f_{\mathrm{\scriptscriptstyle II},j}$'s would have to be just right to violate (\ref{eq:dc}).
It is unlikely even under adversarial attacks, and hence, for practical purposes, those in Figure~\ref{fig:dc} are approximate don't-care conditions.
For empirical verification, we measured the fractal classifier for MNIST 4-9 on the training images: because they are partitioned, we know which image is assigned to which don't care condition.
66.6\% of training images of 4 are classified robustly: 37.9\% are due to the 1st-level exact don't care condition, and 28.7\% are due to the condition illustrated in Figure~\ref{fig:dc}(b).
63.2\% of training images of 9 are classified robustly: 45.8\% are due to the condition illustrated in Figure~\ref{fig:dc}(a), and 17.4\% are due to another condition where all 1st-level and 3rd-level classifiers are correct.

If a fractal classifier could satisfy a don't-care condition or an approximate don't-care condition for every training datum, it would be 100\% robust on the training set.
This could be achieved if each training subset from fractal divide could be classified robustly by an L2NNN, and detailed discussion is in the appendix.
In practice, however, not all training subsets can be classified robustly, -- even after fractal divide, some subtasks of robust classification are still too difficult for a monolithic L2NNN.
In addition, there is generalization gap in robustness.
For example, the fractal classifier for MNIST 4-9 is 64.9\% robust on the training set and 60.2\% robust on the test set.

\section{Experiments} \label{sec:results}

Pre-trained fractal models are at this dropbox:\\
{\small\url{https://www.dropbox.com/sh/qpgy297vok7xwxb/AACeo3Ih-cYWheS7AuFfa0jba}}

\begin{table}[t]
\caption{Accuracies on natural and adversarial test images of 4 and 9 where the $L_2$-norm limit of distortion is 2.}
\label{tbl:49}
\centering
\begin{tabular}{lcccccc}
\toprule
                & Natural & PGD    & BA     & CW     & SCW    & Best-attack \\
\midrule
Vanilla          & 99.7\% & 48.8\% &    0\% &    0\% &    0\% &    0\% \\
\citeauthor{mit} & 98.6\% & 97.8\% & 10.6\% & 47.8\% & 17.6\% &  1.3\% \\
\citeauthor{kolter} & 98.9\% & 97.7\% & 15.9\% & 60.4\% & 28.3\% & 12.0\% \\
Adv training ($L_2$ $\varepsilon=2$) &  98.4\% & 87.2\% & 57.4\% & 52.7\% & 52.6\% & 52.1\% \\
L2NNN            & 99.1\% & 94.4\% & 42.7\% & 41.3\% & 41.3\% & 41.3\% \\
NBR              & 99.1\% & 92.5\% & 57.2\% & 55.5\% & 55.3\% & 55.3\% \\
Fractal          & 99.1\% & 97.4\% & 62.0\% & 97.8\% & 70.3\% & 60.2\% \\
\bottomrule
\end{tabular}
\end{table}

\begin{table}[t]
\caption{Accuracies on natural and adversarial test images of ``pullover'' and ``coat'' where the $L_2$-norm limit of distortion is 1.5.}
\label{tbl:24}
\centering
\begin{tabular}{lcccccc}
\toprule
                & Natural & PGD    & BA     & CW     & SCW    & Best-attack \\
\midrule
Vanilla          & 92.0\% &  3.5\% &    0\% &  0.6\% &    0\% &    0\% \\
Adv training ($L_\infty$ $\varepsilon=8/255$) &  91.0\% & 54.3\% & 25.4\% & 19.2\% & 18.7\% & 18.6\% \\
Adv training ($L_2$ $\varepsilon=1.5$)       &  82.3\% & 67.6\% & 55.7\% & 54.0\% & 53.7\% & 53.6\% \\
L2NNN            & 89.6\% & 48.6\% & 30.4\% & 27.6\% & 26.8\% & 26.6\% \\
Fractal          & 89.3\% & 65.1\% & 41.8\% & 75.1\% & 50.9\% & 39.6\% \\
\bottomrule
\end{tabular}
\end{table}

\begin{table}[t]
\caption{Comparison of aggregation methods on 4-9 and ``pullover''-``coat''.}
\label{tbl:ablation}
\centering
\begin{tabular}{lcccc}
  \toprule
        & \multicolumn{2}{c}{4-9} & \multicolumn{2}{c}{``pullover''-``coat''} \\
        & Natural & Best-attack & Natural & Best-attack \\
\midrule
Fractal              & 99.1\% & 60.2\% & 89.3\% & 39.6\% \\
Markov random field  & 99.0\% & 55.8\% & 83.0\% & 30.8\% \\
Gaussian naive Bayes & 99.0\% & 56.4\% & 87.8\% & 35.7\% \\
\bottomrule
\end{tabular}
\end{table}

\subsection{Evaluation setup} \label{sec:setup}

As discussed in Section~\ref{sec:bg}, the appropriate $L_2$-norm limit for robustness measurement is 2 for MNIST and 1.5 for Fashion-MNIST.
Measurement is by running four attacks: projected gradient descent (PGD) \cite{mit}, boundary attack (BA) \cite{boundary}, Carlini \& Wagner (CW) attack \cite{carlini} and seeded CW (SCW).
Foolbox \cite{foolbox} is used for PGD and BA; CW is original code from \cite{carlini}; SCW is a CW variant that starts its search from a transfer attack.
Iteration limit is 100 for PGD, 50K for BA, and 10K for CW and SCW.
For the transfer attack that seeds SCW, we derive a surrogate model from a fractal classifier by reducing $s_i$ parameters in (\ref{eq:x}) to around 5 and attack the surrogate with CW; the rationale is to get around any vanishing-gradient problem that CW might have.
We do the same for SCW on an NBR model, and for SCW on other competitors we simply use the same seeds as the fractal model.
A classifier is considered robust on an image if it remains correct under all four attacks.

The MNIST models from \cite{mit,kolter,l2nnn,nbr} are publicly available.
For Table~\ref{tbl:49}, we simply use their two logits for 4 and 9 to form binary classifiers.
The adversarially trained models are built by adapting the code of \cite{mit}, details in the appendix.


\subsection{Binary classifiers}

Let's start with 4 versus 9, the most challenging pair of digits in MNIST, and ``pullover'' versus ``coat'', which is one of the challenging pairs of labels in Fashion-MNIST.

The fractal 4-9 classifier has three levels and 17 subtasks: 6 in the first level, 5 in the second and 6 in the third.
The accuracies of subtask classifiers on natural images range from 66.8\% to 92.9\%.
None of them is a good classifier, and yet the ensemble classifier is both accurate and robust.

The fractal ``pullover''-``coat'' classifier has three levels and 23 subtasks: 9 in the first level, 8 in the second and 6 in the third.
The accuracies of subtask classifiers range from 52.7\% to 64.8\%.
This again shows the effectiveness of fractal aggregation in building a reliable ensemble out of highly unreliable knowledge sources.
The L2NNN model in Table~\ref{tbl:24} is trained as if it's a single subtask that has been assigned all training data with the two labels.

Table~\ref{tbl:ablation} present empirical comparisons between fractal aggregation and other ensemble methods, by using Markov random fields (MRF) and Gaussian naive Bayes (GNB) to combine the same subtask classifiers.
Both MRF and GNB use (\ref{eq:x}) as features.
GNB needs no training, and MRF's weight parameters are trained with the loss function in \cite{nbr} for training $b_i$'s.
As mentioned, fractal aggregation like (\ref{eq:nest1})(\ref{eq:nest2}) can be viewed as a nested formula of and's and or's with uncertainty.
MRF and GNB are capable of emulating one-level fuzzy-and or fuzzy-or but are unable to express the fine structures in fractal aggregation.

\begin{table}[h]
\caption{Accuracies on natural and adversarial test images of MNIST where the $L_2$-norm limit of distortion is 2.}
\label{tbl:mnist}
\centering
\begin{tabular}{lcccccc}
\toprule
                & Natural & PGD    & BA     & CW     & SCW    & Best-attack \\
\midrule
Vanilla          & 99.1\% & 52.9\% &    0\% &    0\% &    0\% &    0\% \\
\citeauthor{mit} & 98.5\% & 97.1\% &  9.4\% & 59.5\% & 26.5\% &  4.8\% \\
\citeauthor{kolter} & 98.8\% & 97.0\% & 18.0\% & 70.4\% & 42.7\% & 13.8\% \\
Adv training ($L_2$ $\varepsilon=2$) &  98.7\% & 92.0\% & 76.8\% & 74.1\% & 72.9\% & 69.8\% \\
L2NNN            & 98.2\% & 94.4\% & 82.0\% & 66.5\% & 66.3\% & 66.2\% \\
NBR-45           & 99.0\% & 96.7\% & 83.2\% & 97.5\% & 58.0\% & 57.9\% \\
Fractal          & 99.0\% & 97.2\% & 81.5\% & 97.6\% & 76.5\% & 70.0\% \\
\bottomrule
\end{tabular}
\end{table}

\begin{table}[h]
\caption{Accuracies on natural and adversarial test images of Fashion-MNIST where the $L_2$-norm limit of distortion is 1.5.}
\label{tbl:fashion}
\centering
\begin{tabular}{lcccccc}
\toprule
                & Natural & PGD    & BA     & CW     & SCW    & Best-attack \\
\midrule
Vanilla          & 90.6\% & 29.1\% &  3.4\% &    0\% &    0\% &    0\% \\
Adv training ($L_\infty$ $\varepsilon=8/255$) &  91.3\% & 60.6\% & 51.7\% & 32.1\% & 30.5\% & 29.7\% \\
Adv training ($L_2$ $\varepsilon=1.5$)       &  85.2\% & 73.3\% & 70.1\% & 63.5\% & 63.0\% & 62.2\% \\
NBR-45           & 90.1\% & 63.0\% & 62.0\% & 63.7\% & 41.9\% & 41.6\% \\
Fractal          & 90.0\% & 77.8\% & 66.8\% & 81.8\% & 63.8\% & 54.5\% \\
\bottomrule
\end{tabular}
\end{table}

\subsection{Full MNIST and Fashion-MNIST classifiers}

Tables~\ref{tbl:mnist} and \ref{tbl:fashion} present the full multiclass classification results on MNIST and Fashion-MNIST.
There are no NBR classifiers available and therefore we build NBR-45 models as additional baselines in the tables.
An NBR-45 classifier has one L2NNN subtask for each pair of labels and hence has 45 subtasks.
Fractal classifiers differ from NBR-45 in that certain subtasks are replaced by fractal structures, and the chosen ones are those that classify challenging pairs of labels.
For MNIST, we replaced the four subtasks for 3-5, 3-8, 4-9 and 7-9, and each is replaced by multiple subtasks that form a fractal binary classifier.
Consequently, the number of subtasks in the overall fractal MNIST classifier is increased to 104.
For Fashion-MNIST, we replaced six subtasks: ``t-shirt/top''-``shirt'', ``pullover''-``coat'', ``pullover''-``shirt'', ``coat''-``shirt'', ``sandal''-``sneaker'' and ``sneaker''-``ankle boot''.
The number of subtasks in the overall fractal Fashion-MNIST classifier is increased to 107.
The subtask classifiers that are not replaced are identical between NBR-45 and fractal classifiers.
Simply put, the comparison between NBR-45 and fractal classifiers in both tables represent the effect of fractal divide and fractal aggregation.


Tables~\ref{tbl:mnist} and \ref{tbl:fashion} show that the fractal classifiers achieve excellent robustness with low cost to natural accuracy.
The MNIST model in \cite{towards} is not differentiable and cannot be evaluated in the same way as in Table~\ref{tbl:mnist}.
Instead we compare with their reported numbers: natural accuracy is 99\% and robust accuracy is 80\% with $L_2$ distortion limit of 1.5.
For $L_2$ limit of 1.5, the fractal MNIST classifier has a robust accuracy of 85.9\%.
Since the fractal model is differentiable, our 85.9\% is likely after more scrutiny from attacks than their 80\%.


\begin{table}[t]
\caption{Predicted robustness as a function of assumed upper bound of local Lipschitz constants.}
\label{tbl:provable}
\centering
\begin{tabular}{lcccc}
  \toprule
       & \multicolumn{2}{c}{MNIST (radius 2)} & \multicolumn{2}{c}{Fashion-MNIST (radius 1.5)} \\
Bound  & Prediction & Correctness & Prediction & Correctness \\
\midrule
1      &  2.1\% &  100\% &  8.5\% & 100\%  \\
0.9    &  7.2\% &  100\% & 13.5\% & 100\%  \\
0.8    & 17.7\% &  100\% & 20.0\% & 100\%  \\
0.7    & 36.9\% & 99.8\% & 28.2\% & 99.9\%  \\
0.6    & 57.4\% & 96.6\% & 38.2\% & 98.2\%  \\
0.5    & 73.7\% & 89.0\% & 49.2\% & 93.0\%  \\
\bottomrule
\end{tabular}
\end{table}

\begin{wrapfigure}{r}{0.35\columnwidth}
\centering
\includegraphics[width=0.15\columnwidth]{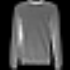}
\includegraphics[width=0.15\columnwidth]{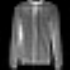}
\caption{Original image and the adversarial image that fools the ``pullover''-``coat'' classifier but not the multiclass classifier.}
\label{fig:adv}
\end{wrapfigure}

An interesting observation is that an adversarial example for a binary classifier is not necessarily one for a multiclass classifier.
Out of the 949 adversarial examples found by boundary attack on the fractal ``pullover''-``coat'' classifier, the full Fashion-MNIST classifier gives the correct answer on 99 of them.
This may seem counterintuitive because the very same binary classifier is inside the multiclass classifier.
Figure~\ref{fig:adv} shows one such example: the original image is a pullover, and boundary attack adds a distortion with $L_2$-norm of 1.43 that seems to mimic a zipper.
The binary classifier's $w$ values are -0.89 for ``pullover'' and -1.58 for ``coat'' on the original, and both become -1.39 on the adversarial example.
However, the ``coat''-``shirt'' fractal ensemble calculates a $w$ value of -1.62 for ``coat'' on the adversarial example, and consequently, in the full multiclass classifier, ``pullover'' is still the label with the least evidence against it.
This creates more obstacles for attacks and is a benefit from the simple scheme of (\ref{eq:multiclass}).

Finally, let us quantify the robustness guarantees of fractal models.
Due to the bounded Lipschitz constant of L2NNN components and the don't care conditions of fractal aggregation, every prediction comes with a robustness $L_2$ radius.
Recall that we can increase/decrease L2NNN outputs by an adversarial offset to simulate the worst-case scenario of an attack, and hence a lower bound of the robustness radius around an input can be measured by the threshold of adversarial offset that changes the classification.
The actual robustness radius depends on local Lipschitz constants \cite{hein} and for L2NNNs they can be substantially below 1 \cite{l2nnn}.
If we know an upper bound of local Lipschitz constants, the robustness radius is no less than the threshold of adversarial offset divided by the upper bound.
Table~\ref{tbl:provable} shows the predicted robustness levels by assuming different upper bound values.
We measure the correctness of these predictions by comparing the list of test images with predicted robust classification against the test images from Tables~\ref{tbl:mnist} and \ref{tbl:fashion} on which fractal models actually remain correct after all four attacks.
Table~\ref{tbl:provable} suggests that it is safe to assume the Lipschitz upper bound of 0.7.
With this, the MNIST model has 36.9\% provable robustness and the Fashion-MNIST model has 28.2\%.
Note that such guarantees do not exist for models based on adversarial training.

\section{Conclusions and future work}

This paper proposes fractal divide and fractal aggregation for robust classification.
Together they form an ensemble approach that achieves better robustness-accuracy trade-off than monolithic models and has a unique property that ensures robustness under certain conditions.
The techniques are demonstrated on MNIST and Fashion-MNIST with new state of the art in adversarial robustness by $L_2$ metric.
Substantial potential exists in improving subtask classifiers, including adversarial training and better nonexpansive architectures.
They are also key in expanding to other metrics and to other applications.

\nocite{fashion}
\nocite{lecun1998mnist}
\nocite{mitpackage}

\bibliographystyle{plainnat}
\bibliography{fractal}

\appendix

\numberwithin{equation}{section}

\section{Proof of Lemma 1}

\begin{proof}
  If $n=1$, the inequality is an equality and is true.
  If $n=2$, the difference between the two sides is
  \begin{align}
    & \frac{1-\alpha_1}{1+\alpha_1} \cdot \frac{1-\alpha_2}{1+\alpha_2} - \frac{1-\alpha_1-\alpha_2}{1+\alpha_1+\alpha_2} \nonumber\\
    =& \frac{\left( 1-\alpha_1 \right)\cdot\left( 1-\alpha_2 \right)\cdot\left( 1+\alpha_1+\alpha_2 \right)}
    {\left( 1+\alpha_1 \right)\cdot\left( 1+\alpha_2 \right)\cdot\left( 1+\alpha_1+\alpha_2 \right)}
     - \frac{\left( 1+\alpha_1 \right)\cdot\left( 1+\alpha_2 \right)\cdot\left( 1-\alpha_1-\alpha_2 \right)}
    {\left( 1+\alpha_1 \right)\cdot\left( 1+\alpha_2 \right)\cdot\left( 1+\alpha_1+\alpha_2 \right)} \nonumber\\
    =& \frac{2\alpha_1^2\alpha_2+2\alpha_1\alpha_2^2}
    {\left( 1+\alpha_1 \right)\cdot\left( 1+\alpha_2 \right)\cdot\left( 1+\alpha_1+\alpha_2 \right)} \nonumber\\
    \geq& 0 \nonumber
  \end{align}
  Therefore, Lemma 1 holds true for $n=2$.

  Suppose Lemma 1 holds true for $n \leq k$.
  For $n = k+1$, let's first apply Lemma 1 on the first $k$ factors on the left-hand side, and then apply Lemma 1 with $n=2$:
  \begin{align}
    \prod_{j=1}^{k+1} \frac{1-\alpha_j}{1+\alpha_j} &\geq \frac{1-\sum_{j=1}^k\alpha_j}{1+\sum_{j=1}^k\alpha_j} \cdot \frac{1-\alpha_{k+1}}{1+\alpha_{k+1}} \nonumber\\
    &\geq \frac{1-\sum_{j=1}^{k+1}\alpha_j}{1+\sum_{j=1}^{k+1}\alpha_j} \nonumber
  \end{align}
  Therefore, Lemma 1 also holds true for $n = k+1$.
  By induction, Lemma 1 is true for any $n$.
\end{proof}

\section{Robustness on training set}

If a fractal classifier could satisfy a don't-care condition or an approximate don't-care condition for every training datum, it would be 100\% robust on the training set.
This could be achieved if each training subset from fractal divide could be classified robustly by an L2NNN.

Let's state the (approximate) don't-care conditions more formally.
For an input with label I, a condition is that one subtask classifier in an odd level classifies it robustly and that all subtask classifiers in the preceding even levels classify it robustly; if the total number of levels is even, a condition can also be that all subtask classifiers in all even levels classify it robustly.
For an input with label II, a condition is that one subtask classifier in an even level classifies it robustly and that all subtask classifiers in the preceding odd levels classify it robustly; if the total number of levels is odd, a condition can also be that all subtask classifiers in all odd levels classify it robustly.

As discussed in Section 3.2, any pair of training data with opposite labels appears in one and only one training subset.
Therefore, a training datum ${\bf t}$ with label I is contrasted against every training datum with label II, and the contrast is distributed among a set of subtasks.
In Figure 1, this set of subtasks correspond to a vertical stack of rectangles that collectively cover a vertical line at the x-coordinate that corresponds to ${\bf t}$.
The stack varies for different ${\bf t}$, but it always matches one of the don't-care conditions: it is composed of either one subtask in an odd level and all subtasks in the preceding even levels, or, if the total number of levels is even, all subtasks in all even levels.
If every training subset from fractal divide is classified robustly by the corresponding L2NNN, then this stack is a set of subtask classifiers that have ${\bf t}$ in their training subset and hence classify ${\bf t}$ robustly.
Therefore, one of the don't-care conditions is satisfied for any ${\bf t}$ with label I.
The argument for label II is similar.

In practice, not all training subsets can be classified robustly, -- even after fractal divide, some subtasks of robust classification are still too difficult for a single L2NNN.
In addition, there is generalization gap in robustness.
For example, the fractal classifier for MNIST 4-9 is approximately 64.9\% robust on the training set and 60.2\% robust on the test set.

\section{Implementation details of adversarial training}

There are no publicly available pre-trained MNIST or Fashion-MNIST classifiers that are adversarially trained with an $L_2$ adversary, or Fashion-MNIST classifier that are adversarially trained with an $L_\infty$ adversary.
So we built these baselines ourselves by adapting the code of [12].
For the $L_2$ adversary, we modified the PGD part of the code of [12] to be the TensorFlow equivalent of the corresponding PyTorch code in [7].
The $\varepsilon$ values used are listed in the tables in Section 4.
The step sizes are $2.5\cdot\varepsilon /\text{number-of-steps}$ as in [7].
All other hyperparameters and settings are identical to the original code of [12].

\end{document}